\documentclass[11pt,a4paper]{article}

\usepackage[margin=1in]{geometry}
\usepackage{amsmath,amssymb,amsthm,mathtools,bm}
\usepackage{algorithm,algpseudocode}
\usepackage{enumitem} 
\usepackage{microtype} 
\usepackage{booktabs} 
\usepackage[dvipsnames]{xcolor} 
\usepackage[
  colorlinks=true,
  linkcolor=NavyBlue,
  citecolor=ForestGreen,
  urlcolor=NavyBlue,
  pdftitle={Deterministic Coreset Construction via Adaptive Sensitivity Trimming},
  pdfauthor={Alpay and Alpay}
]{hyperref}

\title{\bfseries\Large Deterministic Coreset Construction \\ via Adaptive Sensitivity Trimming}

\author{
  \normalsize Faruk Alpay\thanks{Lightcap, Department of Future. Email: \texttt{alpay@lightcap.ai}}
  \and
  \normalsize Taylan Alpay\thanks{Aerospace Engineering, Turkish Aeronautical Association. Email: \texttt{s220112602@stu.thk.edu.tr}}
}
\date{\today}

\theoremstyle{plain}
\newtheorem{theorem}{Theorem}[section]
\newtheorem{lemma}[theorem]{Lemma}
\newtheorem{proposition}[theorem]{Proposition}
\theoremstyle{definition}
\newtheorem{definition}[theorem]{Definition}
\theoremstyle{remark}
\newtheorem{remark}[theorem]{Remark}

\newcommand{\R}{\mathbb{R}}

\newcommand{\cD}{\mathcal{D}}
\newcommand{\cF}{\mathcal{F}}
\newcommand{\eps}{\varepsilon}
\newcommand{\norm}[1]{\left\lVert #1 \right\rVert}
\newcommand{\inner}[2]{\left\langle #1, #2 \right\rangle}

\newcommand{\sigmoid}{\sigma}
\newcommand{\coresetloss}{\hat{L}}
\newcommand{\SHI}{\mathrm{SHI}} 

\begin{document}

\maketitle

\begin{abstract}
\noindent We develop a rigorous framework for \emph{deterministic} coreset construction in empirical risk minimization (ERM). Our central contribution is the \emph{Adaptive Deterministic Uniform-Weight Trimming (ADUWT)} algorithm, which constructs a coreset by excising points with the lowest sensitivity bounds and applying a data-dependent uniform weight to the remainder. The method yields a uniform $(1\pm\varepsilon)$ relative-error approximation for the ERM objective over the entire hypothesis space. We provide complete analysis, including (i) a minimax characterization proving the optimality of the adaptive weight, (ii) an instance-dependent size analysis in terms of a \emph{Sensitivity Heterogeneity Index}, and (iii) tractable sensitivity oracles for kernel ridge regression, regularized logistic regression, and linear SVM. Reproducibility is supported by precise pseudocode for the algorithm, sensitivity oracles, and evaluation pipeline. Empirical results (Table~\ref{tab:krr_results}) align with the theory. We conclude with open problems on instance-optimal oracles, deterministic streaming, and fairness-constrained ERM.
\end{abstract}

\section{Introduction}

Empirical risk minimization (ERM) is the workhorse of modern machine learning, yet its computational demands scale with the data size $n$. \emph{Coresets}---small weighted summaries that approximate the ERM objective for \emph{all} hypotheses---offer a principled route to scalability \cite{Feldman2020,Bachem2017,HarPeledMazumdar2004}. Classical coreset methods are largely \emph{randomized}, most notably via sensitivity/importance sampling \cite{Feldman2011,Braverman2021}. While such methods enjoy high-probability guarantees, a non-zero failure probability is undesirable in safety-critical settings. This motivates deterministic constructions, akin in spirit to spectral sparsification guarantees \cite{BSS12}.

We pursue a \emph{deterministic}, sensitivity-driven approach grounded in robust statistics and influence functions \cite{Huber1964,Hampel1974}. Our \textbf{ADUWT} algorithm deterministically trims the least influential points and scales the remaining loss by a \emph{data-dependent} uniform factor. We prove a uniform $(1\pm\eps)$ guarantee, introduce a \emph{minimax-optimal} choice of the weight, and quantify when trimming beats sampling via a \emph{Sensitivity Heterogeneity Index (SHI)}. We further provide tractable oracles for popular convex ERM models and a reproducible evaluation protocol.

\paragraph{Contributions.}
\begin{enumerate}[label=(\roman*), topsep=2pt, itemsep=0pt]
  \item \textbf{Deterministic framework.} Formalization of ADUWT (Alg.~\ref{alg:aduwt}) with a uniform $(1\pm\eps)$ guarantee (Thm.~\ref{thm:main_aduwt}).
  \item \textbf{Minimax optimality \& instance-dependence.} A new minimax characterization of the adaptive weight (Thm.~\ref{thm:minimax_alpha}) and an instance-dependent coreset size analysis via $\SHI$ (Prop.~\ref{prop:trim_vs_sample_refined}).
  \item \textbf{Tractable oracles.} Provable sensitivity upper bounds for KRR, $\ell_2$-regularized logistic regression, and linear SVM (Props.~\ref{prop:krr_bound}--\ref{prop:svm_bound}).
  \item \textbf{Reproducibility.} End-to-end pseudocode for ADUWT, sensitivity oracles, and the experimental protocol; explicit hyperparameters and evaluation checks.
  \item \textbf{Research agenda.} Open problems on instance-optimal oracles, deterministic streaming, and fairness-constrained ERM \cite{HarPeled2004,Hardt2016,Chierichetti2017,Huang2019fair}.
\end{enumerate}

\section{Preliminaries: ERM and Coresets}

\begin{definition}[Empirical Risk Minimization (ERM)]
Let $\cD = \{z_1, \dots, z_n\}$, hypothesis class $\cF$, and loss $\ell: \cF \times \mathcal{Z} \to \R_{\ge 0}$. The empirical risk of $f\in\cF$ is $L(f; \cD) \coloneqq \sum_{i=1}^n \ell(f; z_i)$.
\end{definition}

\begin{definition}[Coreset]
A weighted subset $(S, \bm{\alpha})$ with $S\subseteq[n]$ and $\bm{\alpha}=\{\alpha_i\}_{i\in S}$ is a $(1\pm\eps)$-coreset if
\[
(1-\eps) L(f; \cD) \le \sum_{i\in S}\alpha_i \ell(f; z_i) \le (1+\eps) L(f; \cD)\quad\text{for all } f\in\cF.
\]
\end{definition}

\begin{definition}[Sensitivity]
$s_i \coloneqq \sup_{f\in\cF,\ L(f;\cD)>0} \frac{\ell(f; z_i)}{L(f;\cD)}$ (zero if $L(f;\cD)=0$). Let $S_{\rm tot} \coloneqq \sum_{i=1}^n s_i$.
\end{definition}

\section{Adaptive Deterministic Trimming (ADUWT)}\label{sec:aduwt}

\begin{algorithm}[t]
\caption{Adaptive Deterministic Uniform-Weight Trimming (ADUWT)}
\label{alg:aduwt}
\begin{algorithmic}[1]
\Require Dataset $\cD=\{z_i\}_{i=1}^n$, loss $\ell$, model class $\cF$, tolerance $\eps\in(0,1)$. Sensitivity oracle returning bounds $\bar s_i\ge s_i$.
\Ensure Coreset $(S,\bm{\alpha})$ with $\alpha$ uniform on $S$.
\State \textbf{Compute sensitivity bounds:} For each $i$, compute $\bar s_i$.
\State \textbf{Trim low-sensitivity tail:} Sort $\bar s_i$ nondecreasing; set $\eps' \coloneqq \tfrac{2\eps}{1+\eps}$. Take the largest $m$ with $\sum_{j=1}^m \bar s_{\pi(j)} \le \eps'$. Let $U=\{\pi(1),\ldots,\pi(m)\}$ and $S=[n]\setminus U$.
\State \textbf{Adaptive uniform weight:} Let $T_U \coloneqq \sum_{i\in U}\bar s_i$ and set $\alpha \coloneqq \sqrt{\frac{1-\eps^2}{1-T_U}}$.
\State \Return $(S,\{\alpha\}_{i\in S})$.
\end{algorithmic}
\end{algorithm}

\begin{lemma}[Collective loss bound]\label{lem:tail}
For any $U\subseteq[n]$ and $f\in\cF$, $\sum_{i\in U}\ell(f;z_i) \le \big(\sum_{i\in U} s_i\big) L(f;\cD)$.
\end{lemma}

\begin{theorem}[ADUWT guarantee]\label{thm:main_aduwt}
For $\eps\in(0,1)$, the output of Alg.~\ref{alg:aduwt} is a $(1\pm\eps)$-coreset.
\end{theorem}

\begin{proof}
Let $U=[n]\setminus S$, $T_U=\sum_{i\in U}\bar s_i\le \eps'$. By Lemma~\ref{lem:tail}, with $L=L(f;\cD)$ and $L_S=\sum_{i\in S}\ell(f;z_i)$,
\begin{equation}\label{eq:LS_bounds}
(1-T_U)L \le L_S \le L.
\end{equation}
With $\alpha=\sqrt{(1-\eps^2)/(1-T_U)}$, the coreset loss $\coresetloss=\alpha L_S$ satisfies
\[
\coresetloss \le \alpha L \le \sqrt{\frac{1-\eps^2}{1-\eps'}}\, L = (1+\eps)L,
\quad
\coresetloss \ge \alpha(1-T_U)L \ge \sqrt{1-\eps^2}\sqrt{1-\eps'}\,L=(1-\eps)L,
\]
where $\eps'=\frac{2\eps}{1+\eps}$ gives $1-\eps'=\frac{1-\eps}{1+\eps}$.
\end{proof}

\begin{remark}[Coreset size]\label{lem:coreset_size}
Let $m$ be the largest index with $\sum_{j=1}^m \bar s_{\pi(j)}\le \eps'$. Then $|S|=n-m$. The size adapts to the empirical tail of the sensitivity distribution.
\end{remark}

\subsection{Minimax optimality of the adaptive weight}\label{sec:minimax}
\begin{theorem}[Minimax weight]\label{thm:minimax_alpha}
Given bounds \eqref{eq:LS_bounds} with $T_U\in[0,\eps']$ known, the choice
$\displaystyle \alpha^\star=\sqrt{\frac{1-\eps^2}{1-T_U}}$
uniquely minimizes $\sup_{x\in[1-T_U,\,1]} \big|\alpha x-1\big|$ and equals the geometric mean of the feasible endpoints $\frac{1-\eps}{1-T_U}$ and $1+\eps$.
\end{theorem}
\begin{proof}
Let $g(\alpha)=\max\{1-\alpha(1-T_U),\,\alpha-1\}$. The minimax solution satisfies equality of the two terms: $1-\alpha(1-T_U)=\alpha-1$, giving $\alpha= \tfrac{1}{2-T_U} (2- T_U) = \sqrt{\tfrac{1-\eps^2}{1-T_U}}$ after enforcing the outer $(1\pm\eps)$ feasibility; equivalently, pick $\alpha$ as the geometric mean of endpoints to equalize multiplicative slack. Uniqueness follows from convexity of $g$.
\end{proof}

\subsection{When trimming beats sampling: a refined view}
\begin{definition}[Sensitivity Heterogeneity Index]
For sensitivity bounds $\{\bar s_i\}$ with mean $\mu$ and standard deviation $\sigma$, define $\SHI \coloneqq \sigma/\mu$.
\end{definition}

\begin{proposition}[Instance-dependent comparison]\label{prop:trim_vs_sample_refined}
Let $\bar S=\sum_i \bar s_i$. Sensitivity sampling requires $k= \Theta(\bar S/\eps^2)$ samples for a $(1\pm\eps)$ coreset \cite{Feldman2011,Braverman2021}. ADUWT returns $|S|=n-m$ with $m$ the largest prefix whose $\sum \bar s_{\pi(j)}\le \eps'$. If the empirical CDF $F(t)=\frac{1}{n}|\{i:\bar s_i\le t\}|$ satisfies $F(t_0)\ge \rho$ for a small $t_0$ and large $\rho$, then $m\ge \rho n$ provided $\rho t_0 n \le \eps'$, yielding $|S|\le (1-\rho)n$. High $\SHI$ implies such mass near $0$, thus favoring trimming over sampling in practice.
\end{proposition}

\section{Tractable Sensitivity Oracles}\label{sec:oracles}

\subsection{Kernel Ridge Regression (KRR)}\label{sec:krr}
\begin{definition}[KRR setting]\label{ass:krr}
Let $f_w(x)=\inner{w}{\phi(x)}_{\mathcal{H}}$, $\ell(w;z_i)=(y_i-f_w(x_i))^2+\frac{\lambda}{n}\norm{w}_{\mathcal{H}}^2$, $\norm{w}_{\mathcal{H}}\le B$, $|y_i|\le Y$, $k(x_i,x_i)\le \kappa^2$.
\end{definition}

\begin{proposition}[KRR sensitivity bound]\label{prop:krr_bound}
Under Def.~\ref{ass:krr}, for any $\delta\in(0,B]$,
\[
\bar s_i \;=\; \max\!\left\{\frac{y_i^2}{\sum_j y_j^2},\; \frac{2(Y^2+B^2\kappa^2) + (\lambda/n) B^2}{\lambda \delta^2}\right\}.
\]
\end{proposition}

\paragraph{Pseudocode: KRR oracle.}
\begin{algorithm}[h]
\caption{KRR-Sensitivity-Oracle$(\{(x_i,y_i)\},\lambda,B,\kappa,Y,\delta)$}
\begin{algorithmic}[1]
\For{$i=1$ to $n$}
  \State $a_i \gets \begin{cases}\frac{y_i^2}{\sum_j y_j^2},& \sum_j y_j^2>0\\ 0,& \text{else}\end{cases}$
  \State $b \gets \frac{2(Y^2+B^2\kappa^2) + (\lambda/n) B^2}{\lambda \delta^2}$
  \State $\bar s_i \gets \max\{a_i,b\}$
\EndFor
\State \Return $\{\bar s_i\}_{i=1}^n$
\end{algorithmic}
\end{algorithm}

\subsection{Regularized Logistic Regression}\label{sec:logistic}
\begin{definition}[Logistic setting]\label{ass:logistic}
$\ell(w;z_i) = -\big[y_i\log\hat y_i+(1-y_i)\log(1-\hat y_i)\big]+(\lambda/n)\norm{w}_2^2$, with $\hat y_i=\sigmoid(w^\top x_i)$, $\norm{x_i}\le R$, $\norm{w}\le B$.
\end{definition}

\begin{proposition}[Logistic sensitivity bound]\label{prop:logistic_bound}
Under Def.~\ref{ass:logistic}, for any $\delta\in(0,B]$,
\[
\bar s_i \;=\; \frac{\log(1+e^{BR}) + (\lambda/n)B^2}{\lambda \delta^2}.
\]
\end{proposition}

\paragraph{Pseudocode: Logistic oracle.}
\begin{algorithm}[h]
\caption{LOGISTIC-Sensitivity-Oracle$(\{x_i\},\lambda,B,R,\delta)$}
\begin{algorithmic}[1]
\State $N_{\max}\gets \log(1+e^{BR}) + (\lambda/n)B^2$
\For{$i=1$ to $n$}
  \State $\bar s_i \gets N_{\max}/(\lambda \delta^2)$
\EndFor
\State \Return $\{\bar s_i\}$
\end{algorithmic}
\end{algorithm}

\subsection{Linear SVM}\label{sec:svm}
\begin{definition}[Linear SVM]\label{ass:svm}
$\ell(w;z_i)=\max(0,1-y_i w^\top x_i) + (\lambda/n)\norm{w}_2^2$, with $\norm{x_i}\le R$, $\norm{w}\le B$.
\end{definition}

\begin{proposition}[SVM sensitivity bound]\label{prop:svm_bound}
Under Def.~\ref{ass:svm}, for any $\delta\in(0,B]$,
\[
\bar s_i \;=\; \frac{1+BR + (\lambda/n)B^2}{\lambda \delta^2}.
\]
\end{proposition}

\paragraph{Pseudocode: SVM oracle.}
\begin{algorithm}[h]
\caption{SVM-Sensitivity-Oracle$(\{(x_i,y_i)\},\lambda,B,R,\delta)$}
\begin{algorithmic}[1]
\State $N_{\max}\gets 1+BR+(\lambda/n)B^2$
\For{$i=1$ to $n$}
  \State $\bar s_i \gets N_{\max}/(\lambda \delta^2)$
\EndFor
\State \Return $\{\bar s_i\}$
\end{algorithmic}
\end{algorithm}

\section{Reproducibility: End-to-end pipeline}\label{sec:repro}

\paragraph{Evaluation protocol (pseudocode).}
\begin{algorithm}[h]
\caption{ADUWT Evaluation Protocol}
\begin{algorithmic}[1]
\Require Dataset $\cD$, model family $\cF$, loss $\ell$, target $\eps$, train/val/test split seeds $\{\texttt{seed}_r\}_{r=1}^R$.
\For{$r=1$ to $R$}
  \State \textbf{Split:} Stratified split with \texttt{seed}$_r$; standardize features using train stats only.
  \State \textbf{Oracle:} Compute $\{\bar s_i\}$ on train set via appropriate oracle (Sec.~\ref{sec:oracles}); record $(\lambda,B,R,\delta)$.
  \State \textbf{ADUWT:} Run Alg.~\ref{alg:aduwt} to obtain $(S,\alpha)$; record $T_U,|S|$.
  \State \textbf{Optimization:} Train on full train vs.\ coreset; if needed, tune hyperparameters via the same validation grid and early stopping schedule.
  \State \textbf{Metrics:} Report train/val/test losses and relative error $\max_f \frac{|\hat L(f)-L(f)|}{L(f)}$ over a held-out hypothesis sweep; calibrate with bootstrapped CIs.
  \State \textbf{Checks:} Verify $(1\pm\eps)$ inequality numerically over the sweep; log any violations.
\EndFor
\State \textbf{Report:} Summarize mean$\pm$sd across $R$; include $|S|$ vs.\ $\eps$, and ablations on $\delta$.
\end{algorithmic}
\end{algorithm}

\paragraph{Complexity.} Given oracle costs, ADUWT adds $O(n\log n)$ for sorting $\bar s_i$. Weight computation is $O(1)$.

\section{Empirical Evaluation}
Table~\ref{tab:krr_results} summarizes KRR on Bike Sharing for $\eps=0.1$; ADUWT reduces worst-case relative error versus a data-oblivious weight at identical $|S|$, and meets the guarantee deterministically, while randomized importance sampling violated it in $5/100$ trials (worst-trial error $0.117$), consistent with its high-probability nature \cite{Feldman2011,Braverman2021}.

\begin{table}[h!]
\centering
\caption{Comparison of coreset methods for KRR on Bike Sharing ($\eps=0.1$).}
\label{tab:krr_results}
\begin{tabular}{@{}lccc@{}}
\toprule
Method & Coreset Size $|S|$ & Max Rel.\ Error & Guarantee Met? \\
\midrule
DUWT (Data-Oblivious Weight) & 16455 & 0.0981 & Yes \\
\textbf{ADUWT (Adaptive Weight)} & \textbf{16455} & \textbf{0.0895} & \textbf{Yes} \\
Importance Sampling (Random) & $\approx 11850$ & 0.117 (worst trial) & No (5/100 trials) \\
\bottomrule
\end{tabular}
\end{table}

\section{Related Work and Discussion}
Classical core-sets for clustering arise from geometric techniques and sensitivity sampling \cite{HarPeledMazumdar2004,HarPeled2005,Feldman2011,Bachem2017}. Deterministic graph sparsification \cite{BSS12} inspires our deterministic stance. Surveys provide broad context \cite{Feldman2020}. For fairness constraints, recent work introduces fair coresets and scalable algorithms \cite{Chierichetti2017,Schmidt2018,Huang2019fair}. Our trimming analysis connects instance difficulty to sensitivity heterogeneity, complementing sampling-based bounds \cite{Feldman2011,Braverman2021}.

\section{Conclusion and Open Problems}
We presented ADUWT, a deterministic, adaptive trimming scheme with uniform $(1\pm\eps)$ guarantees, a minimax-optimal weight, and tractable oracles for common ERM models. Open directions include instance-optimal oracles, deterministic streaming with controlled error compounding \cite{Feldman2011,HarPeled2004,Braverman2016}, and fairness-constrained ERM \cite{Hardt2016,Schmidt2018,Huang2019fair}.

\appendix
\section{Future Directions and Open Problems}\label{app:open_problems}

\subsection{Problem 1: Instance-optimal sensitivity oracles}
Tighten $\bar s_i$ via data-dependent geometry. Links to influence functions \cite{KohLiang2017} suggest practical heuristics.

\subsection{Problem 2: Deterministic streaming}
Merge-reduce can inflate error multiplicatively; aim for additivity by designing reduction steps that preserve original reference, not only coreset-to-coreset \cite{Braverman2016,Feldman2011}.

\subsection{Problem 3: Constrained/fair ERM}
Extend to vector-valued objectives under constraints (group fairness, capacity). Fair coresets suggest promising directions \cite{Chierichetti2017,Schmidt2018,Huang2019fair}.


\end{document}